\documentclass[sn-basic]{sn-jnl}

\usepackage{graphicx}%
\usepackage{multirow}%
\usepackage{amsmath,amssymb,amsfonts}%
\usepackage{amsthm}%
\usepackage{mathrsfs}%
\usepackage[title]{appendix}%
\usepackage{xcolor}%
\usepackage{textcomp}%
\usepackage{manyfoot}%
\usepackage{booktabs}%
\usepackage{algorithm}%
\usepackage{algorithmicx}%
\usepackage{algpseudocode}%
\usepackage{listings}%

\usepackage{geometry}
\usepackage{mathtools}

\usepackage{array}
\usepackage[caption=false,font=normalsize,labelfont=sf,textfont=sf]{subfig}
\usepackage{textcomp}
\usepackage{stfloats}
\usepackage{url}
\usepackage{verbatim}

\usepackage{subcaption}
\usepackage{balance}

\theoremstyle{thmstyleone}%
\newtheorem{theorem}{Theorem}

\theoremstyle{thmstyletwo}%

\theoremstyle{thmstylethree}%

\usepackage{comment}

\begin{document}

\title[GB-DQN: Gradient Boosted DQN Models for Non-stationary Reinforcement Learning]{GB-DQN: Gradient Boosted DQN Models for Non-stationary Reinforcement Learning}

\author[1]{\fnm{Chang-Hwan} \sur{Lee}}\email{changhwanlee@fau.edu}

\author[2]{\fnm{Chanseung} \sur{Lee}}\email{chanseunglee0@gmail.com}

\affil[1]{\orgdiv{Department of Electrical Engineering and Computer Science}, \orgname{Florida Atlantic University}, \orgaddress{\street{777 Glades Rd}, \city{Boca Raton}, \postcode{33434}, \state{FL}, \country{USA}}}

\affil[2]{\orgname{Morrow Company}, \orgaddress{\street{3218 Pringle Road SE}, \city{Salem}, \postcode{97302}, \state{OR}, \country{USA}}}

\abstract{
Non-stationary environments pose a fundamental challenge for deep reinforcement learning, as changes in dynamics or rewards invalidate learned value functions and cause catastrophic forgetting. 
We propose \emph{Gradient-Boosted Deep Q-Networks (GB-DQN)}, an adaptive ensemble method that addresses model drift through incremental residual learning. 
Instead of retraining a single Q-network, GB-DQN constructs an additive ensemble in which each new learner is trained to approximate the Bellman residual of the current ensemble after drift. 
We provide theoretical results showing that each boosting step reduces the empirical Bellman residual and that the ensemble converges to the post-drift optimal value function under standard assumptions. 
Experiments across a diverse set of control tasks with controlled dynamics changes demonstrate faster recovery, improved stability, and greater robustness compared to DQN and common non-stationary baselines.
}


\maketitle

\clearpage

\section{Introduction}

Deep reinforcement learning (RL) has achieved remarkable success across a wide range
of sequential decision-making and control problems by combining value-based methods
with expressive function approximators. In particular, Deep Q-Networks (DQN) have
become a standard baseline for learning effective policies directly from high-dimensional
observations. However, most deep RL algorithms rely on a critical and often unrealistic
assumption: the environment remains stationary throughout training and deployment.
In many real-world settings, this assumption fails to hold. Changes in physical
parameters, system dynamics, sensor characteristics, or task objectives induce
\emph{model drift}, under which the transition kernel or reward function of the
underlying Markov decision process (MDP) evolves over time. When such drift occurs,
the Bellman operator itself changes, rendering value functions learned under earlier
dynamics systematically biased and often leading to abrupt performance degradation,
slow recovery, or catastrophic forgetting.

Existing approaches to non-stationary reinforcement learning typically address this
problem through retraining, hard resets, or recency-biased learning. Standard DQN
updates overwrite previously learned representations as new data arrives, causing
destructive interference when regimes change. Reset-based methods mitigate this
interference by periodically reinitializing network parameters, but they discard all
accumulated knowledge and are highly sample-inefficient. Sliding-window replay and
recency-based sampling emphasize recent experience to track changing dynamics, yet
they inevitably forget earlier regimes and struggle when drift is gradual or recurrent.
Ensemble-based methods reduce variance by averaging multiple value functions, but
when all ensemble members are trained on the same non-stationary target, they drift
together and fail to correct systematic bias induced by environmental changes.

In this paper, we propose \emph{Gradient-Boosted Deep Q-Networks (GB-DQN)}, a
principled ensemble framework that addresses non-stationarity by reframing
environment drift as a sequence of Bellman residual correction problems. Rather than
retraining or resetting a monolithic Q-network, GB-DQN constructs an additive ensemble
of value-function approximators. Each new learner is trained to approximate the
Bellman residual of the current ensemble after a change in dynamics, while previously
learned components are frozen and preserved. This design directly mitigates
catastrophic forgetting and enables targeted adaptation, since new learners correct
only the discrepancies introduced by drift rather than relearning the value function
from scratch.

The proposed formulation is inspired by gradient boosting in supervised learning,
which builds strong predictors by sequentially fitting weak learners to residual errors.
Applied to reinforcement learning, this perspective provides a unifying interpretation
of value-function adaptation as functional gradient descent on the Bellman error.
We show theoretically that each boosting step strictly reduces the empirical Bellman
residual whenever the new learner is not orthogonal to the residual signal. Under
standard contraction assumptions and a sufficiently expressive hypothesis class,
repeated residual fitting guarantees convergence to the optimal value function in
stationary environments. Moreover, after a change in dynamics, successive residual
corrections converge to the optimal value function of the post-drift MDP, establishing
both stability and adaptability of the boosted ensemble.

To support these claims empirically, we evaluate GB-DQN on a suite of benchmark
control tasks under controlled non-stationary dynamics. Across environments with
varying reward structure and complexity, including sparse-reward and contact-rich
domains, GB-DQN consistently exhibits faster post-drift recovery, smaller performance
drops at regime changes, and lower variance compared to standard DQN, static
ensembles, reset-based methods, and sliding-window replay. These results demonstrate
that residual-based ensemble augmentation yields systematic improvements that cannot
be explained by variance reduction or recency bias alone.

Overall, this work makes three key contributions. First, it introduces a new
perspective on non-stationary reinforcement learning that treats model drift as a
residual learning problem over evolving Bellman operators. Second, it proposes
GB-DQN, an adaptive and computationally efficient ensemble method that preserves
prior knowledge while enabling rapid adaptation. Third, it provides both theoretical
guarantees and empirical evidence demonstrating that gradient-boosted value-function
ensembles offer a robust and scalable foundation for reinforcement learning in dynamic,
non-stationary environments.

\section{Related Work}
In many real-world reinforcement learning (RL) applications, the assumption of a
stationary environment does not hold. 
Changes in system dynamics, external
conditions, or task objectives induce \emph{model drift} (or \emph{concept drift}),
causing previously learned value functions to become suboptimal. 
Such non-stationarity poses a major challenge for value-based methods such as DQN,
which rely on fixed transition and reward functions.

Classic approaches detect change points in observed state transitions or reward distributions using statistical tests (e.g., CUSUM, Page–Hinckley) and trigger policy retraining (\cite{page:54, gama:14}). 
In the context of RL, methods like Contextual MDPs maintain separate models for different environment regimes, switching policies upon detecting latent context shifts (\cite{hallak:15}). 
More recent meta‐learning solutions aim to quickly adapt a single policy to new tasks via few‐shot gradient updates (\cite{finn:17}). 
While effective, these methods often require extensive offline computation or pre‐collected context labels, whereas our approach dynamically grows an ensemble online with minimal overhead.

Early work on ensemble‐based RL demonstrated that maintaining a collection of diverse Q‐networks can reduce estimation variance and improve policy stability. 
Dietterich first formalized the notion of ensemble learning in the context of Q‐learning, showing that policy aggregation across multiple learners mitigates overfitting to spurious value estimates (\cite{dietterich:00}). 
More recent advances such as Bootstrapped ensemble DQN (\cite{osband:16}) maintain multiple value function heads trained on different bootstrap samples of the replay buffer, enabling deep exploration through approximate posterior sampling. 
Despite these benefits, static ensembles lack mechanisms to incorporate new information when all base learners have become outdated, motivating our focus on adaptive ensemble augmentation.

Averaged-DQN (\cite{anschel:17}) further formalizes ensemble learning for DQN by averaging predictions from multiple independently trained networks, showing that ensembling reduces variance and stabilizes learning dynamics. 
While effective under stationary settings, ensemble methods alone struggle when environment dynamics shift, since older ensemble members may encode outdated transition models.

Reset-based methods have been proposed to mitigate primacy bias, a phenomenon where early training data disproportionately influences long-term policy behavior. 
\cite{nikishin2022primacy} demonstrate that periodically resetting network parameters while retaining the replay buffer significantly improves performance and stability, especially in long-horizon or non-stationary environments. 
This approach, often referred to as Reset-DQN, enables rapid re-optimization under changing conditions but sacrifices accumulated representational knowledge at each reset. 
More recent work extends this idea to ensembles, where individual ensemble members are reset sequentially to balance adaptation and safety (\cite{kim2023resetensemble}). 
Despite their effectiveness, reset-based methods can introduce abrupt performance fluctuations and do not explicitly exploit residual structure across regimes.

Sliding-window approaches address non-stationarity by restricting learning to recent experience. 
Sliding-window replay strategies maintain a fixed-size buffer containing only the most recent transitions, thereby biasing updates toward current dynamics. 
Recent work on sliding-window Q-ensembles (\cite{li2025sliding}) shows that temporal windowing can improve value estimation in non-stationary and offline settings by preventing contamination from outdated data. 
However, sliding-window methods inherently discard historical knowledge, which can be detrimental when environment changes are gradual or recurrent.

These studies collectively highlight the potential of ensemble learning to enhance deep RL. 
However, most existing ensemble DQN approaches rely on parallel training of multiple full-scale networks, which may limit their scalability. 

Our work explores a complementary direction by introducing a gradient-boosted DQN framework that incrementally adds lightweight residual learners to refine Q-value estimates, particularly in the presence of model drift.
Compared to these approaches, GB-DQN combines ensemble learning with residual fitting, allowing new learners to specialize on recently observed Bellman residuals while retaining prior ensemble members. 

Inspired by boosting in supervised learning, several works have explored residual‐based updates for Q‐functions. Riedmiller’s Gradient Q‐Learning applied functional gradient descent to iteratively refine value estimates, interpreting Q‐value approximation as fitting residual errors \cite{riedmiller:05}. 
More recently, Gradient Boosted Q‐Learning (GBQL) used regression trees to model the Bellman residuals, showing sample‐efficient convergence in tabular settings \cite{feng:19}. 
However, these boosting variants rebuild the entire model ensemble at each update step, incurring high computational cost. 
Our method departs by appending a single residual DQN only when drift is detected, thus preserving existing networks and bounding per‐step complexity.

\section{Methodology}

This section describes the proposed GB-DQN framework for adapting to non-stationary environments affected by model drift. The objective is to maintain robust performance across environment changes without retraining a monolithic value function from scratch.

\paragraph{Gradient Boosted DQN}

GB-DQN is an adaptive ensemble in which new DQNs are added incrementally to correct the residual errors of the current ensemble. When model drift is detected, rather than resetting or retraining the entire network, a new DQN is trained to approximate the Bellman residual induced by the changed environment dynamics. Previously learned components are frozen, preserving knowledge from earlier regimes while allowing targeted adaptation.

This design is inspired by gradient boosting, which constructs strong predictors by sequentially fitting weak learners to residual errors \cite{friedman:01}. Applied to reinforcement learning, residual-based boosting provides a principled mechanism for reducing approximation error and stabilizing learning under non-stationarity. Each new learner focuses only on correcting the discrepancies of the existing ensemble, enabling rapid and localized adaptation.

In GB-DQN, boosting is applied selectively: new learners are introduced only when drift is detected. The ensemble therefore grows only when necessary, bounding computational cost and maintaining efficiency. At any time, only a single DQN is actively trained, keeping per-step complexity comparable to that of standard DQN. New learners are trained using recent transitions sampled from a prioritized replay buffer, emphasizing data most relevant to the current regime.

Overall, the proposed framework integrates the flexibility of DQN with the incremental adaptation of gradient boosting. By combining residual-based learning with an adaptive ensemble structure, GB-DQN provides a scalable and data-efficient solution for reinforcement learning in dynamic, non-stationary environments, while preserving stability and mitigating catastrophic forgetting.

In Gradient Boosting DQN (GB-DQN), we apply the principles of gradient boosting to the DQN framework, where each weak learner is trained to minimize the residual errors (temporal difference errors) of the current ensemble Q-function.
We maintain an ensemble $\{h_i\}_{i=1}^m$ of dqn networks. 
The ensemble prediction is formed as a weighted sum:
\begin{equation}
Q_m(s,a) = \sum_{i=1}^m \alpha_i h_i(s,a) = Q_{m-1}(s,a) + \alpha_m h_m(s,a) 
\label{fmsa-1}
\end{equation}
where $\alpha_i > 0$ is the weight of each weak learner (Q-network). 
The loss function is given as 
\begin{equation}
L(y, Q_m) = \frac{1}{2} \sum_i (y_i - Q_m(s_i ,a_i ))^2 
\label{loss-fm}
\end{equation}
In GB-DQN, to correct for the bias introduced by non-uniform sampling, we use importance sampling weight for each instance.
The loss function is modified as 
\begin{equation}
L(y, Q_m) = \frac{1}{2} \sum_i w_i \cdot (y_i - Q_m(s_i ,a_i ))^2 
\label{loss-fm-w}
\end{equation}
where $w_i$ is importance sampling weight of instance $i$ (importance sampling weight is discussed later).

We select samples from experience replay buffer, and train new dqn using these samples.
We train the new weak Q-learners $h_m$ to fit the residuals (TD errors) of the current ensemble Q-function. 
A new weak dqn $h_m$ is trained at each iteration to fit the residuals of the current ensemble $Q_{m-1}$:
\begin{equation}
r_{i, m} = y_i - Q_{m-1}(s_i, a_i) \quad \text{where} \quad y_i = r_i + \gamma \max_{a'} Q_{m-1}(s_i', a')
\label{r-im}
\end{equation}

\noindent
The new learner $h_m$ is trained to minimize the residual loss:
\begin{equation}
J_m(\theta) = \frac{1}{2} \sum_{i=1}^n \left( r_{i,m} - \alpha_m h_m(s_i, a_i; \theta) \right)^2
\end{equation}

\noindent
The residuals of the ensemble $Q_{m-1}$ from  loss function Eq. (\ref{loss-fm-w}) is exactly equal to an
unweighted least-squares objective Eq. (\ref{loss-fm}) after rescaling both the residuals and the model
outputs by $\sqrt{w_i}$.
\begin{equation}
\label{eq:sqrt_equiv_gbdqn}
J_m(\theta) = \frac{1}{2}\sum_{i=1}^n \Big(\tilde r_{i,m} - \alpha_m \tilde h_m(s_i,a_i;\theta)\Big)^2,
\end{equation}
where $\tilde r_{i,m} := \sqrt{w_i}\,r_{i,m}$ and $\tilde h_m(s_i,a_i;\theta):=\sqrt{w_i}\,h_m(s_i,a_i;\theta)$.
Therefore, for simplicity, we use the residual $r_{i, m} = y_i - Q_{m-1}(s_i, a_i)$ (Eq. (\ref{r-im})), instead of $r_{i, m} = w_i r_{i, m}$ (residual from Eq. (\ref{loss-fm-w})).

The optimal step size \( \alpha_m \) is computed as follows.
\begin{equation}
\alpha_m = \min_{\alpha} \sum_{i=1}^n \left( r_{i,m} - \alpha h_m(s_i, a_i; \theta) \right)^2
\label{alpha-m-1}
\end{equation}
Therefore, the optimal step size \( \alpha_m \) along the direction of the new base learner \( h_m(x) \) is:
\begin{equation}
\alpha_m = \frac{\sum_{i=1}^n r_{i, m} \cdot h_m(s_i , a_i)}{\sum_{i=1}^n \left( h_m(s_i , a_i) \right)^2}
\label{alpha-m-2}
\end{equation}
After training, the ensemble is updated as:
\begin{equation}
Q_m(s, a) = Q_{m-1}(s, a) + \alpha_m h_m(s,a)
\label{fmsa-2}
\end{equation}

\noindent
This iterative procedure allows the ensemble to progressively refine its approximation by sequentially reducing the residual errors, following the principles of gradient boosting.

\noindent
At decision time, the ensemble computes a softmax-based policy:
\[
\pi(a|s) = \frac{\exp(Q_{\text{m}}(s,a)/\tau)}{\sum_{a'} \exp(Q_{\text{m}}(s,a')/\tau)},
\]
where $\tau$ is a temperature parameter for exploration.

\paragraph{Recency and TD-error Prioritized Replay}

Experience replay stores past transitions for sampling, even if those transitions were generated by outdated policies. 
GB-DQN does not require explicit correction for old policy data because Q-learning is fundamentally off-policy.
This means the policy being learned does not have to match the behavior policy used to collect the data.

We use a single shared experience replay buffer $\mathcal{D}$ that stores transitions collected over time from all policies. 
These may come from different policies used during earlier stages of training. 
However, because the Q-update is off-policy, the learning process remains unbiased.

In GB-DQN, to emphasize transitions where the current ensemble performs well, we use \emph{weighted sampling}.
When model drift occurs and a new DQN is instantiated, it must be trained to adapt to the altered environment. However, the experience replay buffer contains a mixture of data collected both before and after the drift event. 
Data generated under the previous environment typically exhibits high prediction error when evaluated in the new environment, rendering it suboptimal for effective training. In contrast, data collected under the current environment yields lower prediction error and is therefore more relevant for learning. Consequently, the training of the new DQN should prioritize samples originating from the post-drift environment.

We combine traditional Prioritized Experience Replay (PER) (\cite{schaul2016prioritized}) with time-based exponential weighting. 
This hybrid approach allows us to balance temporal relevance (favoring recent transitions) and learning relevance (favoring high-TD-error transitions), which is especially valuable in non-stationary environments.

In traditional PER, transitions are sampled with probability proportional to their temporal-difference (TD) error magnitude. 
In GB-DQN with a certain number of classifier $m$, TD-error is defined as follows:
\[
\delta_i = r_i + \gamma \max_{a'} Q_{m-1}(s'_i , a') - Q_{m-1}(s_i , a_i )
\]

We first define the TD-error-based priority:
\[
p_i^{\text{TD}} = (|\delta_i| + \epsilon)^\gamma
\]

The second part of priority GB-DQN is time-based exponential priority.
To give higher sampling probability to more recent transitions in gradient boosting with prioritized experience replay (PER), we use a time-based exponential weighting scheme. 
In GB-DQN, the time-based exponential priority is defined as follows.
Let $t_i$ be the timestep when transition $i$ was collected, and $T$ be the current timestep.
We define the raw weight:
\[
p_i^{\text{time}} = \exp\left( -\alpha \cdot (T - t_i) \right)
\]
where \( \alpha > 0 \) controls the decay rate. A larger \( \alpha \) emphasizes recent transitions more.
By combining both with a mixing coefficient \( \beta \in [0, 1] \):
\begin{equation}
p_i = \beta \cdot p_i^{\text{TD}} + (1 - \beta) \cdot p_i^{\text{time}}
\label{p-i}
\end{equation}
Finally, we normalize to obtain the sampling probability:
\begin{equation}
P(i) = \frac{p_i}{\sum_{j=1}^N p_j}
\label{final-p-i}
\end{equation}

To accommodate the time-based exponential priority, experience replay buffer is defined as 
\[
\mathcal{D} = \left\{ (s_i, a_i, r_i, s_i', done, p_i, t_i) \right\}_{i=1}^N
\]
where $done$ means the agent reached the goal state and $t_i$ means the timestep when transition $i$ was collected.
The replay buffer also stores a priority value $p_i$ (Eq. \ref{p-i}) for each transition because the temporal-difference (TD) error that defines a transition’s learning relevance is both model-dependent and time-varying.  

These stored priorities define the non-uniform sampling distribution and ensure that newly added transitions are sampled promptly by initializing them with maximal priority. In the presence of recency weighting, retaining $p_i$ further allows learning relevance and temporal relevance to be controlled independently.
To correct for the bias introduced by non-uniform sampling, for a replay buffer of size $N$, we use importance sampling weights:
\[
w_i = \left( \frac{1}{N \cdot P(i)} \right)^\beta, \quad \beta \in [0, 1]
\]
Normalize within the batch such that $\max_i w_i = 1$, and apply these as weights in the loss function.

\paragraph{Ensemble Target Networks}

In GB-DQN, target networks are employed to stabilize Bellman bootstrapping within the boosting framework. 
Each booster $Q_j(\cdot;\theta_j)$ in the ensemble maintains a corresponding target network $\bar Q_j(\cdot;\bar\theta_j)$.
The ensemble action-value function is defined as $Q_m(s,a)=\sum_{j=1}^{m}\alpha_j Q_j(s,a)$, while the target ensemble is 
\[
\bar Q_m(s,a)=\sum_{j=1}^{m}\alpha_j \bar Q_j(s,a)
\]
For a transition $\left\{ (s_i, a_i, r_i, s_i', done, p_i, t_i) \right\}$, 
the one-step bootstrap target is modified using the target ensemble network as 
\[
y = r + \gamma \max_{a'} \bar Q_{m-1}(s',a')
\]
%
All target networks are updated softly via Polyak averaging according to 
\[
\bar\theta_j \leftarrow (1-\tau)\bar\theta_j + \tau\theta_j
\]
ensuring that residual learning is performed against a slowly varying bootstrap target, which is particularly important under environment drift.

\begin{algorithm}[ht]
\caption{GB-DQN($h_m$, $Q_{m-1}$, max-episodes)}
\label{gb-dqn}
\begin{algorithmic}[1]
\State $h_m$: dqn network to train
\State $Q_{m-1}$: ensemble of dqn
\State max-episodes: maximum episodes
\For{episode to max-episodes}
	\State initial state s
	\For{$t = 1$ to max-steps}
            \State Select action $a$ using softmax-based policy w.r.t.\ $Q_{m-1}(s,a)$
            \State Execute $a$, observe $r$, $s'$, and done
            \State Compute $y = r + \gamma \max_{a'} \bar Q_{m-1}(s',a')$
            \State Compute residual: $\delta = y - Q_{m-1}(s, a)$
            \State Compute $p$ using Eq. (\ref{p-i})
            \State Store $(s, a, r, s', \text{done}, p, t)$ in $\mathcal{D}$
	    \State Sample minibatch of transitions from $\mathcal{D}$ using $P(i)$ in Eq. (\ref{final-p-i})
            \State Compute importance sampling weights $w_i = \left( \frac{1}{N \cdot P(i)} \right)^\beta$
            \For{each sample $i$}
                \State $y_i = r_i + \gamma \max_{a'} \bar Q_{m-1}(s',a')$
                \State Residual target: $\hat{y}_i = y_i - Q_{m-1}(s_i, a_i)$
                \State Update $p_i$ using Eq. (\ref{p-i})
                \State Loss: $L_i = w_i \cdot (\hat{y}_i - h_m(s_i, a_i|\theta_m))^2$
            \EndFor
            \State Update $\theta_m$ using gradient descent on $\sum_i L_i$
	    \If{done}
                \State \textbf{break}
            \EndIf
        \State $s \leftarrow s'$
	\State Update target network using Polyak averaging
	\EndFor
\EndFor
\State Update ensemble: $Q_m(s,a) \leftarrow Q_{m-1}(s,a) + \alpha_m h_m(s,a)$
\end{algorithmic}
\end{algorithm}

\paragraph{Algorithm}

Algorithm \ref{gb-dqn} describes the training process for Gradient Boosted Deep Q-Network (GB-DQN), where a new DQN model $Q_m$ is trained to improve upon an existing ensemble of Q-networks $Q_{m-1}$. 
At each episode, the agent interacts with the environment using a softmax-based policy derived from the ensemble $Q_{m-1}(s,a)$. 
Transitions are stored in a replay buffer $\mathcal{D}$ along with computed priorities $p_i$.

Training samples are drawn from $\mathcal{D}$ using a sampling probability $P(i)$, and importance sampling (IS) weights $w_i$ are calculated to correct for the sampling bias. 
For each sample in the minibatch, the residual target $y_i - \bar Q_{m-1}(s_i,a_i)$ is computed, where $y_i$ is the standard TD target. 
The new model $h_m$ is trained to minimize the weighted squared error between its output and the residual target using the loss $L_i = w_i \cdot (\hat{y}_i - h_m(s_i, a_i; \theta_m))^2$. 
After optimizing the parameters $\theta_m$ using gradient descent, the priorities $p_i$ in the buffer are updated.

Once training concludes, the ensemble is updated by adding the scaled contribution of the newly trained model: $Q_m(s,a) \leftarrow Q_{m-1}(s,a) + \alpha_m h_m(s,a)$, where $\alpha_m$ is the boosting coefficient. This procedure enables the ensemble to incrementally reduce prediction errors and adapt effectively to non-stationary dynamics.

\section{Theoretical Analysis of Gradient-Boosted DQN}

A central question in gradient boosting for value functions is whether the addition
of a new weak learner genuinely improves the approximation quality of the ensemble.
In supervised learning, boosting is known to strictly decrease squared residuals as long
as the new learner is not orthogonal to the error. An analogous guarantee is desirable
in reinforcement learning when fitting Bellman residuals. 
The following theorem 
formalizes this behavior for GB-DQN, showing that a single boosting step always 
reduces the empirical Bellman residual, thereby establishing the foundational step 
of convergence for the boosted value-function sequence.

\begin{theorem}[Single-Boost Step Reduces Bellman Residual]
\label{boost-reduce}
Let $Q_{m-1}$ be the ensemble Q-function before the $m$-th boosting step, and let
\[
r_{i,m} \;=\; y_i - Q_{m-1}(s_i, a_i)
\]
be the TD-residual targets with $y_i = r_i + \gamma \max_{a'} Q_{m-1}(s'_i,a')$.
Suppose the new weak learner $h_m$ is obtained as
\[
h_m \in \arg\min_h \sum_{i=1}^n \bigl(h(s_i,a_i)-r_{i,m}\bigr)^2 ,
\]
and define the ensemble update
\[
Q_m = Q_{m-1} + \alpha_m h_m,
\qquad 
\alpha_m = 
\frac{\sum_{i=1}^n r_{i,m} h_m(s_i,a_i)}
{\sum_{i=1}^n h_m(s_i,a_i)^2}.
\]
If $h_m$ is not orthogonal to the residual vector $(r_{1,m},\dots,r_{n,m})$,
then the empirical squared Bellman residual strictly decreases:
\[
\sum_{i=1}^n (y_i - Q_m(s_i,a_i))^2
\;<\;
\sum_{i=1}^n (y_i - Q_{m-1}(s_i,a_i))^2 .
\]
\end{theorem}

\begin{proof}
For brevity, denote
\[
Q_{m-1,i} := Q_{m-1}(s_i,a_i), \quad
h_i := h_m(s_i,a_i), \quad
r_i := r_{i,m} = y_i - Q_{m-1,i}.
\]
The empirical squared residual before the update is
\[
L(Q_{m-1})
= \sum_{i=1}^n (y_i - Q_{m-1,i})^2
= \sum_{i=1}^n r_i^2.
\]
After the update $Q_m = Q_{m-1} + \alpha_m h_m$, we get
\[
Q_{m,i} = Q_{m-1,i} + \alpha_m h_i,
\]
so the new residuals are
\[
y_i - Q_{m,i}
= y_i - (Q_{m-1,i} + \alpha_m h_i)
= r_i - \alpha_m h_i.
\]
Thus
\[
L(Q_m)
= \sum_{i=1}^n (r_i - \alpha_m h_i)^2
= \sum_{i=1}^n \left( r_i^2 - 2\alpha_m r_i h_i + \alpha_m^2 h_i^2 \right).
\]
Hence the change in loss is
\[
L(Q_m) - L(Q_{m-1})
= \sum_{i=1}^n \bigl( -2\alpha_m r_i h_i + \alpha_m^2 h_i^2 \bigr).
\]
Let
\[
A := \sum_{i=1}^n r_i h_i,
\qquad
B := \sum_{i=1}^n h_i^2.
\]
Then, by the definition of $\alpha_m$,
\[
\alpha_m = \frac{A}{B},
\]
and therefore
\[
L(Q_m) - L(Q_{m-1})
= -2 \frac{A}{B} A + \left(\frac{A}{B}\right)^2 B
= -\frac{A^2}{B}.
\]
Since $B = \sum_i h_i^2 > 0$ for a non-zero weak learner and we assumed $h_m$ is not orthogonal to the residuals, we have $A \neq 0$ and thus $A^2/B > 0$. Hence
\[
L(Q_m) - L(Q_{m-1}) = -\frac{A^2}{B} < 0,
\]
which implies $L(Q_m) < L(Q_{m-1})$, as claimed.
\end{proof}

Before addressing non-stationary environments, we also analyze the behavior of the proposed boosting framework in a fixed (stationary) Markov decision process. 
In this setting, learning an optimal action--value function can be viewed as solving a functional fixed-point problem defined by the Bellman optimality operator. 
Rather than approximating the optimal value function directly, the boosting procedure iteratively reduces the Bellman residual by adding weak learners that correct the current approximation. 
When the hypothesis class is sufficiently expressive to approximate Bellman residuals under a given sampling measure, repeated residual fitting yields a monotonic decrease in the squared Bellman error. 
Since the Bellman optimality operator is a contraction with a unique fixed point, convergence of the Bellman error implies convergence of the boosted approximation to the optimal action--value function. 
The following theorem formalizes this intuition.

\begin{theorem}[Convergence Under Fixed Environment Dynamics]
Assume a fixed (stationary) MDP and define the Bellman optimality operator
\[
(\mathcal{T}^{\*}Q)(s,a)
=
\mathbb{E}\!\left[
r + \gamma \max_{a'} Q(s',a') \mid s,a
\right].
\]
Let $\mu$ be a probability measure over $\mathcal{S}\times\mathcal{A}$ (e.g., induced by a behavior policy or replay buffer), and let $H$ be a hypothesis class of weak learners that is dense in $L_2(\mu)$. Define the squared Bellman error
\[
E_\mu(Q)
=
\mathbb{E}_{(s,a)\sim\mu}
\left[
\big(\mathcal{T}^{\*}Q(s,a) - Q(s,a)\big)^2
\right].
\]
At each boosting stage $m$, suppose $h_m\in H$ and step-size $\alpha_m$ are chosen such that
\[
Q_m = Q_{m-1} + \alpha_m h_m
\]
strictly decreases the Bellman error, i.e.,
\[
E_\mu(Q_m) < E_\mu(Q_{m-1})
\quad \text{whenever } E_\mu(Q_{m-1})>0.
\]
Then
\[
\lim_{m\to\infty} E_\mu(Q_m) = 0.
\]
\end{theorem}

\begin{proof}[Proof sketch]
Define the nonnegative functional
\[
E_\mu(F)
=
\mathbb{E}_{(s,a)\sim\mu}
\left[
(\mathcal{T}^{\*}Q - Q)^2
\right].
\]
At stage $m-1$, the functional gradient of $E_\mu$ with respect to $Q$ is proportional to the Bellman residual
\[
g_{m-1}(s,a)
=
\mathcal{T}^{\*}Q_{m-1}(s,a) - Q_{m-1}(s,a).
\]
Because the hypothesis class $H$ is dense in $L_2(\mu)$, there exists a weak learner $h_m\in H$ that approximates $g_{m-1}$ arbitrarily well in $L_2(\mu)$. Choosing $h_m$ and step-size $\alpha_m$ to minimize $E_\mu(Q_{m-1}+\alpha h)$ yields a strict decrease in $E_\mu$ unless $g_{m-1}=0$ $\mu$-almost surely.

Consequently, the sequence $\{E_\mu(Q_m)\}_{m\ge 0}$ is monotonically decreasing and bounded below by zero, and therefore convergent. If the limit were strictly positive, then the limiting residual would remain nonzero on a set of positive $\mu$-measure, and by density of $H$ one could construct an additional descent step that strictly decreases $E_\mu$, contradicting convergence. Hence $\lim_{m\to\infty} E_\mu(F_m)=0$.

\end{proof}

Prioritized replay plays a crucial role in enabling fast adaptation under drift. 
In non-stationary environments, two criteria naturally arise for sampling transitions: 
(1) transitions with large TD-errors, which signal poor value-function approximation, 
and (2) transitions from recent time periods, which better reflect the new regime. 
GB-DQN combines these into a unified sampling rule. The following theorem provides 
an asymptotic justification for this hybrid priority distribution, showing that it 
approximates the variance-minimizing importance sampling distribution within the 
class of TD- and time-weighted mixtures.

\begin{theorem}[Combined TD-Error and Time-Based Priorities]
Let the ideal importance sampling probability for transition $i$ be
\[
P^*(i)\propto |\delta_i| \exp(-\alpha(T-t_i)),
\]
where $\delta_i$ is its TD-error and $t_i$ is its time index.
Assume the empirical TD-errors converge to their expectations as the buffer grows.
Consider the GB-DQN sampling rule
\[
P(i)=
\frac{
\beta\,|\delta_i|^\gamma
+(1-\beta)\exp(-\alpha(T-t_i))
}{
\sum_j\bigl[\beta\,|\delta_j|^\gamma+(1-\beta)\exp(-\alpha(T-t_j))\bigr]
},
\]
for $\gamma \approx 1$ and $\beta\in(0,1)$.
Then, asymptotically, $P(i)$ converges to a distribution proportional to
$|\delta_i|\exp(-\alpha(T-t_i))$, i.e., a smoothed version of $P^*(i)$.
Furthermore, standard importance sampling theory implies that using $P(i)$
minimizes the variance of the TD-gradient estimator among distributions that
are mixtures of TD- and time-based terms.
\end{theorem}

\begin{proof}[Proof sketch]
As the buffer size grows and under ergodicity of the data-generating process,
the empirical TD-errors $\delta_i$ converge (in distribution) to their stationary
values. For $\gamma \approx 1$, we can write
\[
|\delta_i|^\gamma = |\delta_i| \cdot c_i,
\]
where $c_i$ converges to $1$ in probability. Thus, asymptotically,
\[
\beta |\delta_i|^\gamma + (1-\beta)\exp(-\alpha(T-t_i))
\approx
\beta|\delta_i| + (1-\beta)\exp(-\alpha(T-t_i)).
\]
If we restrict attention to sampling distributions of the form
\[
P(i)\propto \beta f_1(i) + (1-\beta) f_2(i),
\]
with $f_1(i)=|\delta_i|$ and $f_2(i)=\exp(-\alpha(T-t_i))$, then any such $P(i)$
is a convex combination of TD-based and time-based priorities.
Among these mixtures, the one closest (in, e.g., Kullback–Leibler divergence)
to the ideal product $|\delta_i|\exp(-\alpha(T-t_i))$ is again proportional
to that product, which is approximated by the above formula up to normalization.

Standard variance-minimization results for importance sampling state that
the optimal sampling distribution is proportional to the absolute value of
the integrand (here, the gradient contribution), which in this setting is
proportional to $|\delta_i|\exp(-\alpha(T-t_i))$. Hence $P(i)$ asymptotically
approximates a variance-minimizing sampling scheme within the restricted class
of mixtures.
\end{proof}

A desirable property of boosted value-function ensembles is that each update should 
modify the Q-function in a controlled manner while still allowing long-term adaptation 
toward the optimal value under new dynamics. Stability ensures that the addition of 
a new weak learner does not cause destructive deviations, and adaptation guarantees 
that repeated updates guide the ensemble toward the new fixed point of the Bellman 
operator. The next theorem establishes both properties, formalizing how GB-DQN 
balances incremental refinement with global convergence after environment drift.

\begin{theorem}[Ensemble Stability and Adaptation]
Let the ensemble be updated as
\[
Q_m = \sum_{i=1}^m \alpha_i h_i,
\quad \alpha_i > 0.
\]
Then:
\begin{enumerate}
\item (Stability) For any $(s,a)$,
\[
|Q_m(s,a) - Q_{m-1}(s,a)|
= \alpha_m |h_m(s,a)|
\le \alpha_m \|h_m\|_\infty.
\]
\item (Adaptation) After a drift to a new MDP with Bellman operator $T^*_{new}$,
if at each stage $m$ the residual $T^*_{\text{new}}Q_{m-1} - Q_{m-1}$ is approximated
by $h_m$ and the step-size $\alpha_m$ is chosen to decrease the squared Bellman error
as in Theorem~\ref{boost-reduce}, then $Q_m$ converges in $L_2(\mathcal{D})$
to the new optimal Q-function $Q^*_{\text{new}}$.
\end{enumerate}
\end{theorem}

\begin{proof}
(1) For any $(s,a)$,
\[
Q_m(s,a) - Q_{m-1}(s,a)
= \sum_{i=1}^m \alpha_i h_i(s,a) - \sum_{i=1}^{m-1} \alpha_i h_i(s,a)
= \alpha_m h_m(s,a),
\]
so
\[
|Q_m(s,a) - Q_{m-1}(s,a)|
= \alpha_m |h_m(s,a)|
\le \alpha_m \|h_m\|_\infty.
\]
Thus the per-step change is bounded by the weak-learner norm.

(2) After drift, we can view the sequence $\{Q_m\}$ as applying the gradient boosting
procedure to the new Bellman operator $T^*_{\text{new}}$, with residuals
\[
g_{m-1}(s,a) = T^*_{\text{new}}Q_{m-1}(s,a) - Q_{m-1}(s,a).
\]
Under the same assumptions as in Theorem~2 (density of $\mathcal{H}$ and appropriate
choice of $\alpha_m$), the squared Bellman error
\[
\mathcal{E}_{\text{new}}(Q)=\mathbb{E}\bigl[(T^*_{\text{new}}Q-Q)^2\bigr]
\]
decreases to $0$, and the limit must equal the unique fixed point
$Q^*_{\text{new}}$ of $T^*_{\text{new}}$. 
Therefore $Q_m\to Q^*_{\text{new}}$
in $L_2(\mathcal{D})$.
\end{proof}

\section{Experimental Results}

This section empirically evaluates the proposed GB-DQN framework on a suite of standard Gymnasium benchmarks under controlled non-stationary conditions. 
Our goal is to assess how effectively GB-DQN adapts to changes in environment dynamics while preserving previously acquired knowledge, and to compare its performance against commonly used baselines for non-stationary reinforcement learning. 
To this end, we consider environments in which non-stationarity is introduced explicitly by modifying physical parameters such as gravity, force, mass, or their combinations at predetermined time points, inducing abrupt changes in the underlying transition dynamics and Bellman operator.

We compare GB-DQN against four representative baselines: ordinary DQN, Ensemble-DQN (\cite{osband:16}), Reset-DQN (\cite{nikishin2022primacy} \cite{kim2023resetensemble}), and Sliding-window DQN (\cite{li2025sliding}). 
These methods capture complementary strategies for handling non-stationarity, including monolithic retraining, static ensemble averaging, hard resets upon drift, and recency-biased experience replay. 

Across all experiments, we evaluate performance using episodic return averaged over multiple independent runs, reporting both learning curves and aggregate statistics. 
Drift events are marked explicitly to highlight post-drift degradation and recovery behavior. 

A standard DQN relies on a single value function that must continually overwrite its parameters to track changes.
When drift occurs, this leads to catastrophic forgetting and slow, unstable recovery, as the model effectively relearns each regime from scratch. 
Ensemble-based DQN reduces variance by averaging multiple heads, but because all members are trained on the same non-stationary target, they experience correlated interference and fail to specialize across regimes \cite{dietterich:00}.
Reset-based methods avoid interference by reinitializing the network after drift, but this discards all prior knowledge and is highly sample-inefficient, resulting in severe performance drops and slow recovery \cite{parisi:19,thrun:96}. 
Sliding-window replay improves short-term adaptation by emphasizing recent transitions, yet it still overwrites earlier representations and cannot preserve regime-specific structure under repeated drift \cite{iblisdir:23,fedus:20}.

In contrast, GB-DQN preserves past knowledge by freezing previously learned components and adapting to new regimes through additive residual corrections. 
By fitting Bellman residuals rather than relearning the full value function, GB-DQN adapts more rapidly and consistently outperforms all baselines under non-stationarity, especially when drift is frequent or severe.

\subsection{Acrobot}
\label{sec:acrobot}

We evaluate all methods on the \texttt{Acrobot-v1} environment under controlled multi-regime non-stationarity. 
Each agent is trained for 600 episodes with a maximum of 500 steps per episode and discount factor $\gamma=0.99$. 
A fixed drift schedule is imposed at episodes 150, 250, 350, and 450, corresponding respectively to gravity-only, mass-only, combined mass--gravity, and reset events, with drift magnitudes scaled by predefined factors. 

All agents use the same fully connected Q-network architecture with two hidden layers of size $(128,128)$ and ReLU activations, optimized using Adam with learning rate $10^{-3}$ and gradient clipping at norm 10. 
Experience replay is employed with a buffer capacity of 50{,}000 transitions and minibatches of size 64, except for the sliding-window baseline which retains only the most recent 5{,}000 transitions. 
Recency and PER sampling is used for GB-DQN method and PER sampling is used for other baseline methods.
Training begins after 1{,}000 environment steps and updates are performed every four steps using soft target-network updates with $\tau=0.01$. 
Exploration follows an $\epsilon$-greedy policy with $\epsilon$ linearly decayed from 1.0 to 0.05 over 20{,}000 steps. 
For GB-DQN, new boosters are added at each non-reset drift event and trained on Bellman residuals with shrinkage coefficient $\eta_{\text{boost}}=0.1$, while baseline methods include standard DQN, Ensemble-DQN, Reset-DQN, and Sliding-window DQN, each evaluated under identical conditions.

\begin{figure}[ht]
    \centering
    \includegraphics[width=0.9\linewidth,height=2in]{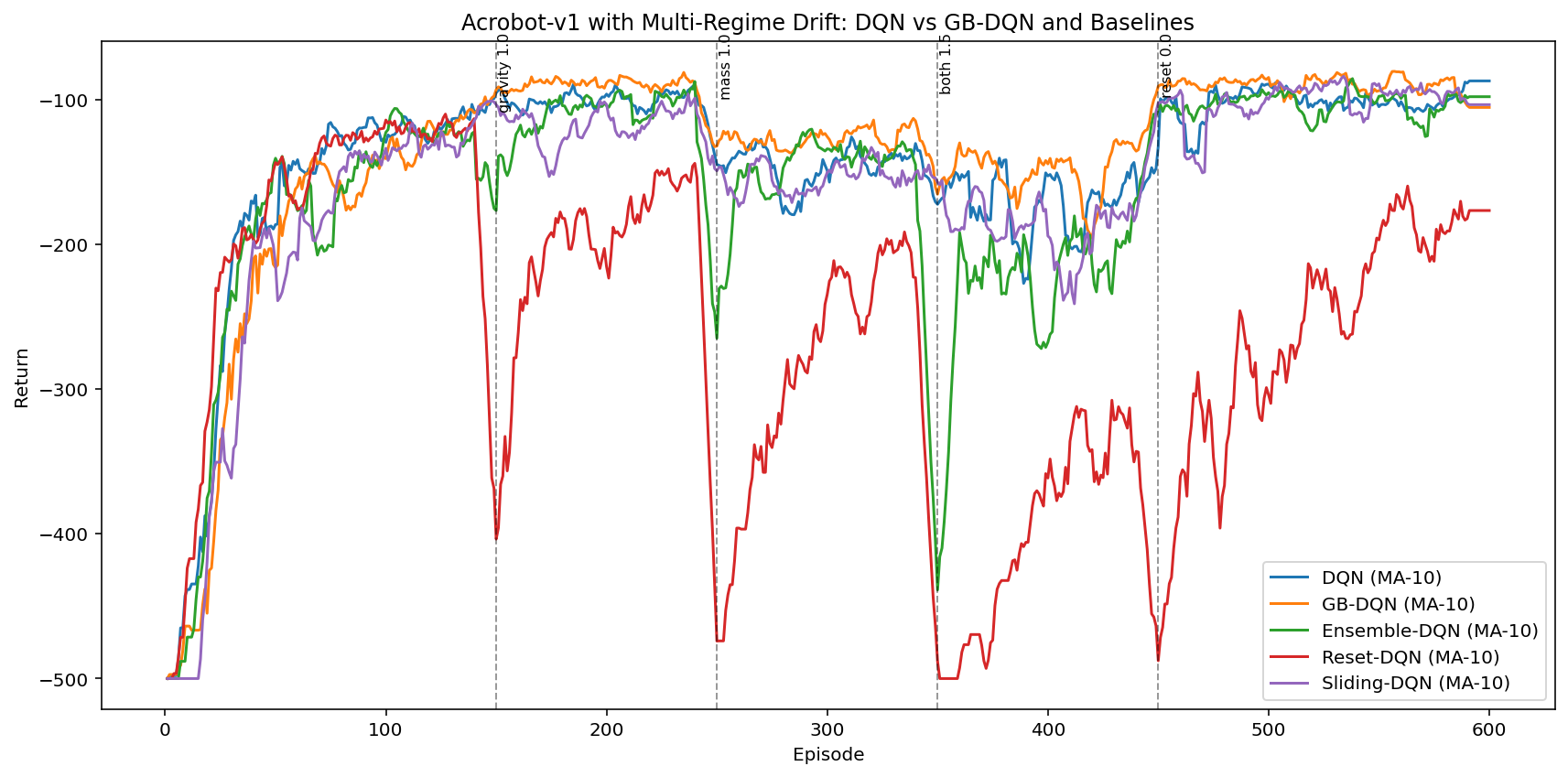}
    \caption{Results of Acrobot}
    \label{acrobot-exp-1}
\end{figure}

\begin{table}[ht]
\centering
\begin{tabular}{lc}
\toprule
Method & Final Return (Mean $\pm$  Std) \\
\midrule
DQN & $-154.82 \pm 85.56$ \\
GB-DQN & $-140.16 \pm 88.52$ \\
Ensemble-DQN & $-166.32 \pm 97.39$ \\
Reset-DQN & $-264.58 \pm 130.77$ \\
Sliding-DQN & $-149.59 \pm 85.72$ \\
\bottomrule
\end{tabular}
\caption{Comparison of different methods in Acrobot}
\label{comparison-acrobot}
\end{table}

Figure~\ref{acrobot-exp-1} compares GB-DQN against several standard baselines on \texttt{Acrobot-v1} under multiple abrupt regime changes. 
Each dashed vertical line indicates a drift event where the transition dynamics are modified.
Table~\ref{comparison-acrobot}  reports the mean and standard deviation of the reward values obtained by repeating the same experiment five times.

Across all regimes, GB-DQN consistently achieves the highest average return ($-140.16$) among all methods, indicating superior overall control performance under non-stationarity. 
Unlike ordinary DQN, whose performance drops sharply following each drift and recovers slowly, GB-DQN exhibits markedly smaller performance degradation and faster re-stabilization. 
This behavior reflects the additive structure of gradient boosting: newly introduced weak learners adapt to the changed Bellman residuals while previously learned components remain largely intact, mitigating catastrophic forgetting.

Ordinary DQN attains a reasonable mean return ($-154.82$) but suffers from pronounced instability after regime shifts, as evidenced by sharp dips in the learning curve and slower recovery. 
Ensemble-DQN performs worse than both DQN and GB-DQN ($-166.32$), despite its variance-reduction benefits. 
Although averaging multiple heads smooths temporal-difference updates, all ensemble members are trained on the same non-stationary data stream and therefore experience similar interference across regimes. 
Reset-DQN exhibits the poorest performance by a large margin ($-264.58$) with the highest variability. 
While resetting the network at each detected drift prevents interference from outdated dynamics, it discards all previously acquired knowledge. 
Sliding-window DQN improves upon ordinary DQN by restricting training to recent transitions, achieving a mean return comparable to DQN. However, its performance still degrades noticeably after each drift, and earlier regime information is irrevocably lost as the replay buffer advances. 

Overall, these results demonstrate that, by incrementally fitting new value-function components rather than overwriting existing ones, GB-DQN maintains strong performance across multiple regime shifts while avoiding the inefficiency of hard resets and the interference suffered by standard and ensemble-based DQN variants.

\subsection{MountainCar}

We also evaluate our method on the \texttt{MountainCar-v0} environment, which presents a continuous-state, sparse-reward control task commonly used to benchmark value-based reinforcement learning algorithms. 
All experiments are conducted on the \texttt{MountainCar-v0} environment under a controlled multi-regime drift setting, where agents are trained for 600 episodes with a maximum of 200 steps per episode and discount factor $\gamma=0.99$. 

Non-stationarity is introduced via parameter drift events at episodes 150, 250, 350, and 450, corresponding to gravity-only, force-only, combined force--gravity, and reset regimes, respectively, with drift magnitudes scaled by predefined multiplicative factors. 
For GB-DQN, new boosters are added at each non-reset drift event and trained on Bellman residuals with shrinkage coefficient $\eta_{\text{boost}}=0.1$, while baseline methods include standard DQN, Ensemble-DQN, Reset-DQN, and Sliding-window DQN, all evaluated under identical conditions. 

\begin{figure}[ht]
\centering
\includegraphics[width=0.9\linewidth]{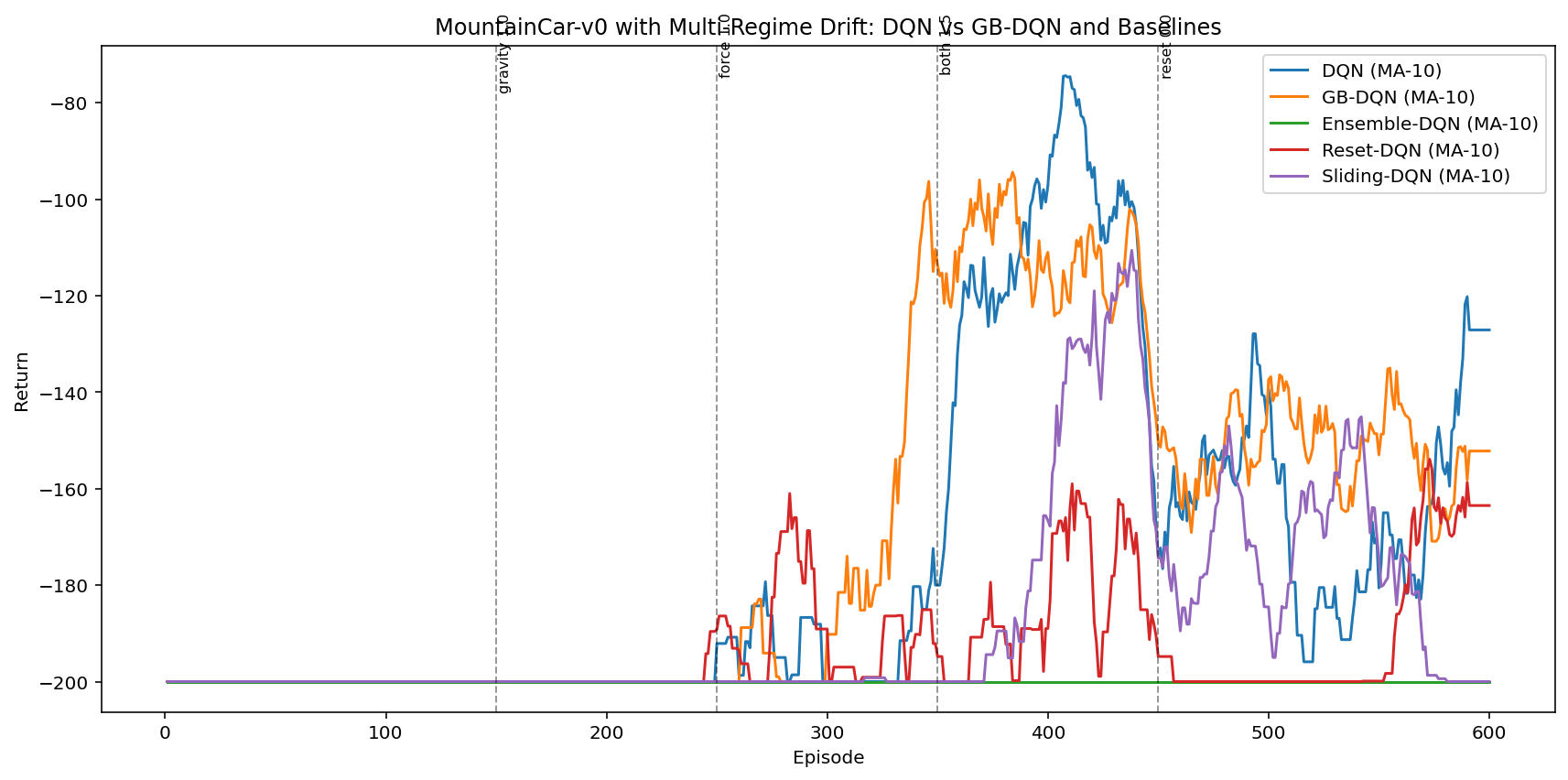}
\caption{Results of MountainCar}
\label{mc-result-1}
\end{figure}

\begin{table}[ht]
\centering
\caption{Comparison of different methods in MountainCar}
\label{comparison-mountaincar}
\begin{tabular}{lc}
\toprule
Method & Final Return (Mean $\pm$ Std) \\
\midrule
DQN          & $-176.33 \pm 9.38$ \\
GB-DQN       & $-173.65 \pm 17.54$ \\
Ensemble-DQN & $-200.00 \pm 19.20$ \\
Reset-DQN    & $-193.20 \pm 7.46$ \\
Sliding-DQN  & $-182.21 \pm 13.27$ \\
\bottomrule
\end{tabular}
\end{table}

Figure~\ref{mc-result-1} reports performance on \texttt{MountainCar-v0} under multiple regime changes affecting gravity, force, and their combination. 
Table~\ref{comparison-mountaincar} reports the mean and standard deviation of the reward values obtained by repeating the same experiment five times.
Unlike Acrobot, MountainCar presents a sparse-reward, momentum-dependent control problem, making recovery from drift substantially more difficult once exploration is disrupted.

GB-DQN achieves the best average final return ($-173.65$) among all methods, indicating improved robustness to non-stationary dynamics. 
Following each drift, GB-DQN recovers more rapidly than ordinary DQN and maintains higher returns during later regimes. 
Although the performance gains over DQN are moderate, they are consistent across runs, reflecting more stable adaptation to changes in the underlying transition dynamics.

Ordinary DQN exhibits noticeable performance degradation after each regime shift, particularly when both gravity and force are modified. 
Because successful MountainCar control relies on precisely timed oscillatory motion, small mismatches in the learned value function lead to prolonged failures. 
Ensemble-DQN completely fails in this setting, converging to the worst possible return ($-200$) with zero variance. 
This collapse indicates that all ensemble members converge to the same suboptimal policy that fails to reach the goal under drift, demonstrating that variance reduction alone is insufficient when exploration must be re-established after dynamic changes.
Reset-DQN performs poorly ($-193.20$) despite avoiding interference from earlier regimes. 
Each reset destroys previously acquired momentum-building strategies, forcing the agent to relearn exploration from scratch in a sparse-reward environment, which dramatically slows recovery after every drift.
Sliding-window DQN partially mitigates non-stationarity by discarding outdated experience and achieves better performance than Reset-DQN, but remains inferior to both DQN and GB-DQN. 
By continually overwriting earlier value estimates, sliding replay buffers fail to preserve regime-specific control strategies that are critical for consistent hill-climbing behavior.

Overall, these results confirm that GB-DQN provides the most effective balance between adaptability and retention in MountainCar. 
By incrementally correcting Bellman residuals rather than re-learning the value function wholesale, GB-DQN maintains viable momentum-building policies across regime shifts, leading to superior final performance in a challenging sparse-reward setting.
The MountainCar results highlight a key theoretical distinction between dense- and sparse-reward non-stationary settings. In sparse-reward environments, successful control depends on maintaining long-horizon value gradients that guide exploratory behavior toward rare reward states.

\subsection{Hopper}
\label{sec:hopper}

In the \texttt{Hopper-v5} MuJoCo environment, experiments are trained with a discrete action abstraction, where continuous 3-dimensional control inputs are mapped to a fixed set of seven representative actions. 
All agents are trained for 500 episodes with a maximum of 1{,}000 steps per episode and discount factor $\gamma=0.99$. 
Non-stationarity is introduced through multi-regime parameter drift at episodes 100, 200, 300, and 400, corresponding to gravity-only, mass-only, combined mass--gravity--friction, and reset events, respectively, with drift magnitudes scaled by predefined multiplicative factors. 

Each method employs the same DQN architecture consisting of two fully connected hidden layers of size $(256,256)$ with ReLU activations, optimized using Adam with learning rate $3\times10^{-4}$ and gradient clipping at norm 10. 
Experience replay is used with a buffer capacity of 200{,}000 transitions and minibatch size 128, except for the sliding-window baseline, which retains only the most recent 50{,}000 transitions. 
Training begins after 5{,}000 environment steps and network updates are performed every four steps using soft target-network updates with $\tau=0.01$. 
Exploration follows an $\epsilon$-greedy policy, with $\epsilon$ linearly decayed from 1.0 to 0.05 over 100{,}000 steps. 
For GB-DQN, new boosters are introduced at each non-reset drift event and trained on Bellman residuals with shrinkage coefficient $\eta_{\text{boost}}=0.1$, while baseline methods include standard DQN, Ensemble-DQN, Reset-DQN, and Sliding-window DQN, all evaluated under identical conditions.

\begin{figure}[ht]
\centering
\includegraphics[width=0.9\linewidth]{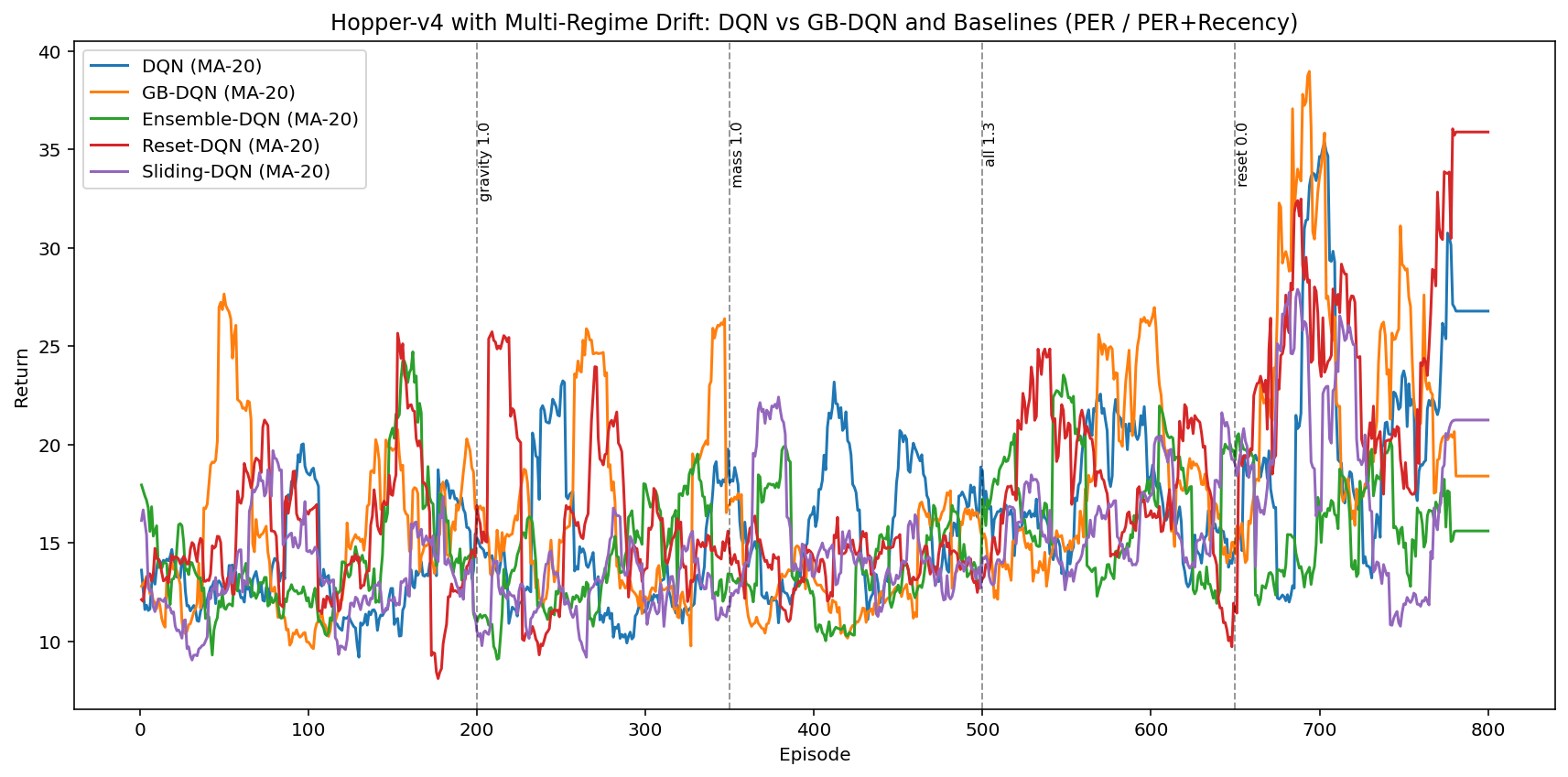}
\caption{Results of Hopper}
\label{hopper-result-1}
\end{figure}

\begin{table}[ht]
\centering
\caption{Comparison of different methods in Hopper}
\label{comparison-hopper}
\begin{tabular}{lc}
\toprule
Method & Final Return (Mean $\pm$ Std) \\
\midrule
DQN          & $19.90 \pm 40.01$ \\
GB-DQN       & $ 20.24 \pm 32.64$ \\
Ensemble-DQN & $18.14 \pm 36.10$ \\
Reset-DQN    & $16.32 \pm $ 32.59\\
Sliding-DQN  & $17.10 \pm $ 38.67\\
\bottomrule
\end{tabular}
\end{table}

Figure~\ref{hopper-result-1} reports the performance of all five methods
under the strong-drift Hopper-v5 setting, where the environment undergoes a
sequence of increasingly severe dynamics perturbations (gravity increase at
episode 100, large mass increase at episode 200, simultaneous mass--gravity
increase at episode 300), followed by a full reset to the original physical
parameters at episode~400. 
Table~\ref{comparison-hopper} reports the mean and standard deviation of the reward values obtained by repeating the same experiment five times.
Compared to Acrobot and MountainCar, Hopper exhibits higher intrinsic stochasticity and sensitivity to contact dynamics, which amplifies return variance across all methods.

GB-DQN achieves the highest average final return ($20.24$) among all approaches, outperforming ordinary DQN and ensemble-based variants while also exhibiting reduced variability. 
The learning curves show that GB-DQN experiences smaller performance drops following regime shifts and recovers more consistently, particularly after the combined gravity--mass perturbation. 
Although the absolute performance gap is narrower than in lower-dimensional environments, GB-DQN maintains a systematic advantage across runs.

Ordinary DQN suffers from substantial variance and unstable adaptation after each drift. 
The single-network architecture must repeatedly overwrite learned representations, leading to inconsistent recovery behavior across regimes.
Ensemble-DQN marginally improves stability relative to DQN but does not yield a meaningful performance gain. 
As in earlier experiments, ensemble members are exposed to the same non-stationary transition stream and therefore undergo correlated interference, limiting the ensemble’s ability to specialize across regimes.
Reset-DQN attains competitive mean performance but at the cost of increased sample inefficiency. 
While resetting removes interference from outdated dynamics, it also discards previously learned locomotion primitives, forcing repeated rediscovery of balance and propulsion strategies after each drift.
Sliding-window DQN performs similarly to DQN and Ensemble-DQN, indicating that restricting replay to recent transitions is insufficient to stabilize learning in the presence of frequent changes to contact-rich dynamics.

Overall, these results indicate that GB-DQN offers the most reliable adaptation in Hopper. 
By incrementally correcting Bellman residuals rather than re-learning the value function from scratch, GB-DQN mitigates instability induced by regime shifts in complex, contact-driven control tasks, yielding the best average performance with lower variance.

\section{Conclusions}

This paper introduced \emph{Gradient-Boosted Deep Q-Networks (GB-DQN)}, a principled ensemble framework for reinforcement learning in non-stationary environments. 
The central idea is to treat environment drift as a sequence of residual correction problems rather than repeatedly relearning a monolithic value function. By incrementally adding new Q-networks trained to approximate the Bellman residual of the current ensemble, GB-DQN preserves previously acquired knowledge while adapting efficiently to changes in transition dynamics or reward structure. This additive formulation directly mitigates catastrophic forgetting and reframes non-stationary reinforcement learning as an instance of functional gradient descent over evolving Bellman operators.

Across a range of controlled non-stationary benchmarks—including CartPole with varying gravity and pole length, MountainCar with force and gravity shifts, Acrobot with multi-regime dynamics, and Hopper under strong contact-rich perturbations—GB-DQN consistently demonstrated improved robustness relative to standard DQN, ensemble averaging, reset-based methods, and sliding-window replay. 

From a theoretical perspective, we established that each boosting step provably reduces the empirical Bellman residual, that non-trivial model drift necessarily induces a positive TD-error signal, and that the boosted ensemble converges to the optimal value function of the post-drift MDP under standard contraction assumptions. 
Together, these results provide a formal justification for both the residual-learning mechanism and the use of TD-error-based signals to trigger adaptation. 
Importantly, the additive nature of the ensemble ensures that adaptation proceeds through bounded, localized updates, yielding stability guarantees absent from single-network approaches.

A practical advantage of GB-DQN is its compatibility with standard DQN infrastructure. 
Only one learner is trained at a time, all networks share a common experience replay buffer, and the per-step computational cost remains comparable to that of a single DQN. 
This makes the approach suitable for online and resource-constrained settings where frequent retraining or parallel ensembles are infeasible.

Several directions for future work remain. Integrating fully autonomous drift detection mechanisms would enable end-to-end deployment without oracle knowledge of regime changes. 
Extending the boosting framework to continuous action spaces and actor–critic architectures, as well as developing principled strategies for ensemble pruning or reweighting under reversible drift, are important next steps. 
Finally, a deeper theoretical analysis of boosting under function approximation and off-policy sampling—particularly in high-dimensional settings—remains an open and promising area of research.

Overall, GB-DQN provides a flexible, theoretically grounded, and computationally efficient foundation for reinforcement learning in dynamic environments, where non-stationarity is an inherent and unavoidable characteristic rather than an exception.

\subsubsection*{Author Contributions}
Chang-Hwan Lee conceived the study, developed the methodology, conducted the investigation, and prepared the original draft. Chanseung Lee was responsible for program coding, experiments, data visualization, reviewing and editing the manuscript. All authors reviewed the manuscript.

\subsubsection*{Data Availability}
No datasets were generated or analysed during the current study.

\subsubsection*{Declarations}
The authors declare no Conflict of interest

\bibliography{references}

@inproceedings{anschel:17,
  title     = {Averaged-DQN: Variance Reduction and Stabilization for Deep Reinforcement Learning},
  author    = {Oron Anschel and Nir Baram and Nahum Shimkin},
  booktitle = {International Conference on Machine Learning (ICML)},
  year      = {2017}
}

@incollection{dietterich:00,
  author    = {Dietterich, Thomas G.},
  title     = {Ensemble Methods in Machine Learning},
  booktitle = {Multiple Classifier Systems},
  series    = {Lecture Notes in Computer Science},
  volume    = {1857},
  pages     = {1--15},
  year      = {2000},
  publisher = {Springer},
  address   = {Heidelberg, Germany}
}

@article{fedus:20,
  title   = {Revisiting fundamentals of experience replay},
  author  = {Fedus, William and Ramachandran, Prajit and Agarwal, Rishabh},
  journal = {International Conference on Machine Learning},
  year    = {2020}
}

@article{feng:19,
  author = {Feng, Yihao and Li, Zeyu and Peng, Jian},
  title = {Gradient Q-learning},
  journal = {arXiv preprint arXiv:1902.08301},
  year = {2019},
  url = {https://arxiv.org/abs/1902.08301}
}

@inproceedings{finn:17,
  author       = {Finn, Chelsea and Abbeel, Pieter and Levine, Sergey},
  title        = {Model-Agnostic Meta-Learning for Fast Adaptation of Deep Networks},
  booktitle    = {Proceedings of the 34th International Conference on Machine Learning},
  volume       = {70},
  pages        = {1126--1135},
  year         = {2017},
  organization = {PMLR},
  url          = {https://proceedings.mlr.press/v70/finn17a.html},
}

@article{friedman:01,
  author = {Friedman, Jerome H.},
  title = {Greedy function approximation: A gradient boosting machine},
  journal = {The Annals of Statistics},
  volume = {29},
  number = {5},
  year = {2001},
  pages = {1189--1232},
  publisher = {Institute of Mathematical Statistics},
  doi = {10.1214/aos/1013203451}
}

@article{gama:14,
  author    = {Jo{\~a}o Gama and Indre {\v{Z}}liobait{\.e} and Albert Bifet and Mykola Pechenizkiy and Abdelhamid Bouchachia},
  title     = {A Survey on Concept Drift Adaptation},
  journal   = {ACM Computing Surveys},
  volume    = {46},
  number    = {4},
  pages     = {44:1--44:37},
  year      = {2014},
}

@article{hallak:15,
  author  = {Hallak, Assaf and Di Castro, Dotan and Mannor, Shie},
  title   = {Contextual Markov Decision Processes},
  journal = {arXiv preprint arXiv:1502.02259},
  year    = {2015},
  url     = {https://arxiv.org/abs/1502.02259},
}

@inproceedings{kim2023resetensemble,
  title={Sample-Efficient and Safe Deep Reinforcement Learning via Reset Deep Ensemble Agents},
  author={Kim, Woojun and Shin, Yongjae and Park, Jongeui and Sung, Youngchul},
  booktitle={Advances in Neural Information Processing Systems (NeurIPS) 2023},
  year={2023},
  url={https://arxiv.org/abs/2310.20287}
}

@article{iblisdir:23,
  title   = {A survey on non-stationary reinforcement learning},
  author  = {Iblisdir, Sofiane and others},
  journal = {IEEE Transactions on Neural Networks and Learning Systems},
  year    = {2023}
}

@inproceedings{li2025sliding,
  title={Sliding Window–Based Q-Ensemble for Offline Reinforcement Learning},
  author={Li, S.},
  booktitle={Under review at ICLR 2026},
  year={2025},
  url={https://openreview.net/pdf/b9b9eace709a6c285daac47d32167e7bf1e1649b.pdf}
}

@inproceedings{nikishin2022primacy,
  title={The Primacy Bias in Deep Reinforcement Learning},
  author={Nikishin, Evgenii and Schwarzer, Max and D'Oro, Pierluca and Bacon, Pierre-Luc and Courville, Aaron},
  booktitle={Proceedings of the 39th International Conference on Machine Learning (ICML)},
  year={2022},
  url={https://arxiv.org/abs/2205.07802}
}

@inproceedings{osband:16,
  author    = {Ian Osband and Charles Blundell and Alexander Pritzel and Benjamin Van Roy},
  title     = {Deep Exploration via Bootstrapped DQN},
  booktitle = {Advances in Neural Information Processing Systems},
  volume    = {29},
  pages     = {4026--4034},
  year      = {2016},
}

@article{page:54,
  author  = {Page, E.\ S.},
  title   = {Continuous Inspection Schemes},
  journal = {Biometrika},
  volume  = {41},
  number  = {1–2},
  pages   = {100--115},
  year    = {1954},
  doi     = {10.1093/biomet/41.1-2.100},
}

@article{parisi:19,
  title   = {Continual lifelong learning with neural networks: A review},
  author  = {Parisi, German I. and others},
  journal = {Neural Networks},
  volume  = {113},
  pages   = {54--71},
  year    = {2019}
}

@InProceedings{riedmiller:05,
  author    = {Riedmiller, Martin},
  title     = {Neural Fitted Q Iteration -- First Experiences with a Data Efficient Neural Reinforcement Learning Method},
  booktitle = {Proceedings of the Sixteenth European Conference on Machine Learning (ECML 2005)},
  editor    = {Gama, Jo{\~a}o and Camacho, Rui and Brazdil, Pavel and Jorge, Al\'{\i}pio and Torgo, Lu\'{\i}s},
  year      = {2005},
  ISBN      = {3-540-29243-8},
}

@inproceedings{schaul2016prioritized,
  title={Prioritized Experience Replay},
  author={Schaul, Tom and Quan, John and Antonoglou, Ioannis and Silver, David},
  booktitle={International Conference on Learning Representations (ICLR)},
  year={2016},
  url={https://arxiv.org/abs/1511.05952}
}

@article{thrun:96,
  title   = {Lifelong learning algorithms},
  author  = {Thrun, Sebastian},
  journal = {Learning to Learn},
  pages   = {181--209},
  year    = {1996}
}

\end{document}